\documentclass[10pt,journal,cspaper,compsoc]{IEEEtran}

\usepackage{graphicx}
\usepackage{amssymb}
\usepackage{amsmath}
\usepackage{amsthm}
\usepackage{thmtools}
\usepackage[T1]{fontenc}

\newtheorem{theorem}{Theorem}[section]
\newtheorem{proposition}{Proposition}[section]

\newtheorem{definition}{Definition}[section]

\begin{document}

\title{\huge Converting Instance Checking to Subsumption: \\ A Rethink for Object Queries over Practical Ontologies}

\author{Jia~Xu,
        Patrick~Shironoshita,
	Ubbo~Visser,
	Nigel~John,
        and~Mansur~Kabuka
\IEEEcompsocitemizethanks{\IEEEcompsocthanksitem J. Xu, P. Shironoshita,  U. Visser, N. John, and M. Kabuka are all with University of Miami, Coral Gables,
FL 33146.\protect\\
Corresponding E-mail: j.xu11@umiami.edu
}
\thanks{}}

%

\IEEEcompsoctitleabstractindextext{%
\begin{abstract}
Efficiently querying Description Logic (DL) ontologies is becoming
a vital task in various data-intensive DL applications. Considered as
a basic service for answering object queries over DL ontologies, instance
checking can be realized by using the most specific concept
(MSC) method, which converts instance checking into subsumption
problems. This method, however, loses its simplicity and efficiency
when applied to large and complex ontologies, as it tends to generate
very large MSC's that could lead to intractable reasoning.  In this
paper, we propose a revision to this MSC method for DL $\mathcal{SHI}$,
allowing it to generate much simpler and smaller concepts that are
specific-enough to answer a given query.  With independence between
computed MSC's, scalability for query answering can also be achieved
by distributing and parallelizing the computations. An empirical
evaluation shows the efficacy of our revised MSC method and the
significant efficiency achieved when using it for answering object queries.
\end{abstract}

}

\maketitle

\IEEEdisplaynotcompsoctitleabstractindextext

%
\IEEEpeerreviewmaketitle

%
%

%
%
%
%

\section{Introduction}
\label{intro}

\IEEEPARstart{D}{escription} logics (DLs) play an ever-growing role in providing a
formal and semantic-rich way to model and represent (semi-) structured
data in various applications, including semantic web, healthcare, and
biomedical research, etc \cite{Horrocks2008}. A knowledge base in
description logic, usually referred to as a DL ontology, consists of
an \emph{assertional} component (ABox $\mathcal{A}$) for data
description, where \emph{individuals} (single objects) are introduced 
and their mutual relationships are described using assertional axioms.
Semantic meaning of the ABox data can then be unambiguously specified
by a \emph{terminological} component (TBox $\mathcal{T}$) of the DL
ontology, where abstract \emph{concepts} and \emph{roles} (binary
relations) of the application domain are properly defined.

In various applications of description logics, one of the core
tasks for DL systems is to provide an efficient way to manage and
query the assertional knowledge (i.e. ABox data) in a DL ontology, 
especially for those data-intensive applications; and DL systems are
expected to scale well with respect to (w.r.t.) the fast growing ABox data, in settings
such as semantic webs or biomedical systems. The most basic reasoning
service provided by existing DL systems for retrieving objects from
ontology ABoxes is \emph{instance checking}, which tests
whether an individual is member of a given concept. Instance
retrieval (i.e. retrieve all instances of a given concept) then can be
realized by performing a set of instance checking calls.

In recent years, considerable efforts have been dedicated to the
optimization of algorithms for ontology reasoning and query answering
\cite{Horrocks1997,Haarslev2008,Motik2007}. However, due to the
enormous amount of ABox data in realistic applications, existing DL
systems, such as HermiT \cite{Motik2007,Motik2009}, Pellet
\cite{Sirin2007}, Racer \cite{Haarslev2001} and FaCT++
\cite{Horrocks1998}, still have difficulties in handling the large
ABoxes, as they are all based on the \emph{(hyper)tableau} algorithm
that is computationally expensive for expressive DLs (e.g. up to
EXPTIME for instance checking in DL $\mathcal{SHIQ}$), where the
complexity is usually measured in the size of the TBox, the ABox and
the query \cite{Calvanese2007,Donini2007,Glimm2008,Ortiz2008,Tobies2001}.
In practice, since the TBox and the query are usually much smaller
comparing with the ABox,  the reasoning efficiency could be mostly
affected by the size of the ABox.

One of the solutions to this reasoning scalability problem is to develop a
much more efficient algorithm that can easily handle large amount of
ABox data. While another one is to reduce size of the data by either 
partitioning the ABox into small and independent fragments that can be
easily handled in parallel by existing systems
\cite{Guo2006,Wandelt2012,Xu2013}, or converting the ABox
reasoning into a TBox reasoning task (i.e. ontology reasoning without
an ABox), which could be \emph{``somewhat''} independent of the data
size,  if the TBox is static and relatively simple, as demonstrated in
this paper.
     
A common intuition about converting instance checking into a TBox
reasoning task is the so-called most specific concept (MSC) method
\cite{Donini1994,Donini2007,Nebel1990} that computes the MSC of
a given individual and reduces any instance checking of this individual
into a \emph{subsumption test} (i.e. test if one concept is more
general than the other). More precisely, for a given individual,
its most specific concept should summarize all information of the 
individual in a given ontology ABox, and should be specific enough to be
subsumed by any concept that the individual belongs to. Therefore,
once the most specific
concept $C$ of an individual $a$ is known, in order to check if $a$ is 
instance of any given concept $D$, it is sufficient to test if $C$ is
subsumed by $D$. With the MSC of every individual in the ABox,  the
efficiency of online object queries can then be boosted by performing an
offline classification of all MSC's that can pre-compute many
instance checks \cite{Donini2007}. Moreover, if a large ontology ABox
consists of data with great diversity and isolation, using the MSC
method for instance checking could be more efficient than the
original ABox reasoning, since the MSC would have the tableau
algorithm to explore only the related information of the given
individual, potentially restricted to a small subset of the
ABox. Also, this method allows the reasoning to be parallelized and
distributed, since MSC's are independent of each other and each
preserves complete information of the corresponding individual. 

Despite these appealing properties possessed by the MSC method, 
the computation of a MSC could be difficult even for a very
simple description logic such as $\mathcal{ALE}$. The difficulty 
arises mainly from the support of qualified existential restrictions
(e.g. $\exists R.C$) in DLs, such that when converting a role assertion
(e.g. $R(a,b)$) of some individual into an existential restriction,
information of that given individual may not be preserved completely. For a
simple example, consider converting assertions 
\begin{center}
$R(a,a)$ $\quad$ and $\quad A(a)$
\end{center}
into a concept for individual $a$. In this case, we can always find a
more specific concept for $a$ in the form of 
\begin{center}
$A \sqcap \underbrace{\exists R. \exists R. \cdots \exists R}_n.A$
\end{center}
by increasing $n$, and none of them would capture the complete
information of individual $a$. Such information loss
is due to the occurrence of cycles in the role assertions, and none of the
existential restrictions in DL could impose a circular interpretation
(model) unless \emph{nominals} (e.g. $\{a\}$) are involved or (local)
reflexivity is presented \cite{Motik2009}. 

Most importantly, due to the support of existential restrictions, 
computation of the MSC for a given individual may involve assertions
of other entities that are connected to it through role assertions.
This implies not only the complexity of the computation for MSC's
but also the potential that the resulting MSC's may have larger than
desired sizes. In fact, in many of the practical ontology ABoxes
(e.g. a social network or semantic webs), most
of the individuals could be connected to each other through role assertions,
forming a huge connected component in the ABox graph. Under this
situation, the resulting MSC could be extremely large and reasoning
with it may completely degenerate into an expensive reasoning procedure.

\begin{figure}[t]
\centering
\includegraphics[width=0.34\textwidth]{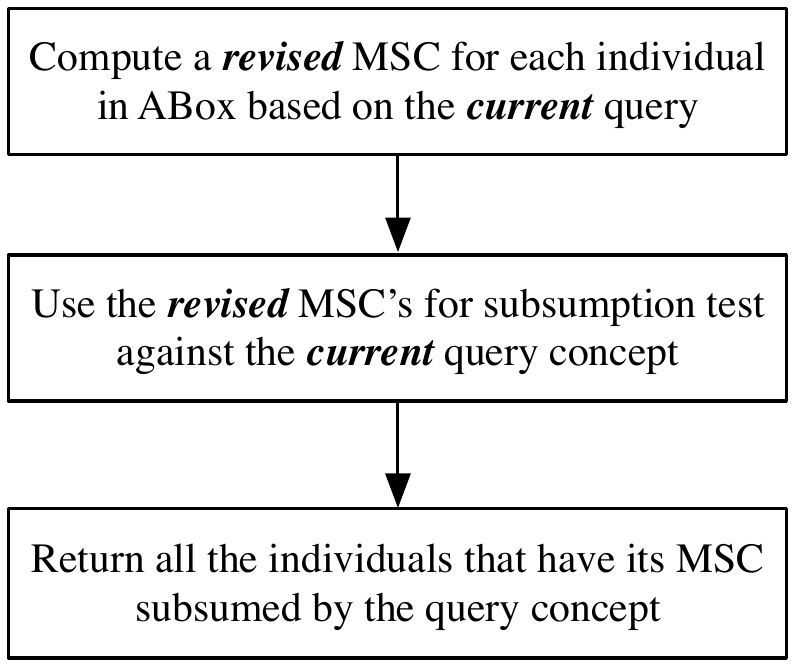}
\caption{A procedure for instance retrieval for a given query based on our revised MSC method.}
\label{ir-proc}
\end{figure}

In this paper, we propose a revised MSC method for DL $\mathcal{SHI}$
that attempts to tackle the above mentioned problems, by applying a \emph{call-by-need}
strategy together with optimizations. That is, instead of computing
the most specific concepts that could be used to answer any queries in
the future, the revised method takes into consideration only the
related ABox information with \emph{current} query and computes a
concept for each individual that is only \emph{specific enough} to
answer it w.r.t. the TBox. Based on this strategy, the revision allows the
method to generate much simpler and smaller concepts than the original
MSC's by ignoring irrelevant ABox assertions. On the other hand,
the complexity reduction comes with the price of re-computation
(i.e. online computation of MSC's)
for every new coming query if no optimization is applied. Nevertheless, as
shown in our experimental evaluation, the simplicity achieved could be
significant in many practical ontologies, and the overhead is thus
negligible comparing with the reasoning efficiency gained for each
instance checking and query answering. Moreover, due to the
re-computations, \emph{we do not assume a static ontology or query},
and the ABox data is amenable to frequent modifications,
such as insertion or deletion, which is in contrast to the original MSC
method where a relatively static ABox is assumed. A procedure for
instance retrieval based on our revised MSC method is shown in Figure
\ref{ir-proc}.

The revised MSC method could be very useful for efficient instance checking
in many practical ontologies, where the TBox is usually small and manageable
while the ABox is in large scale as a database and tends to change
frequently. Particularly, this method
would be appealing to large ontologies in \emph{non-Horn} DLs, where
current optimization techniques such as rule-based reasoning or
pre-computation may fall short. Moreover, the capability to
parallelize the computation is another compelling reason to use this
technique, in cases where answering object queries may demand thousands
or even millions of instance checking tasks.

Our contributions in this paper are summarized as follows:
\begin{enumerate} 
\item  We propose a call-by-need strategy for the original MSC method
  that, instead of computing the most specific concepts offline to handle 
  any given query, it allows us to focus on the current queries
  and to generate online much smaller concepts that are sufficient
  to compute the answers.

  This strategy makes our MSC method suitable for
  query answering in ontologies, where
  frequent modifications to the ontology data are not uncommon.

\item We propose optimizations that can be used to further reduce sizes of
  computed concepts in practical ontologies for more efficient
  instance checking. 

\item Finally,  we evaluate our approach on a range of test ontologies
  with large ABoxes, including ones generated by existing benchmark
  tools and realistic ones used in biomedical research.
\end{enumerate} 

The evaluation shows
 the efficacy of our proposed approach that can generate significantly
 smaller concepts than the original MSC. It also shows the great
 reasoning efficiency that can be achieved when using the revised
 MSC method for instance checking and query answering.

The rest of the paper is organized as follows: in Section \ref{preliminaries}, we
introduce the the background knowledge of a description logic and DL 
ontology; in Section \ref{msc_method}, we give more detailed
discussion about the MSC method and our call-by-need strategy; Section 
\ref{method} presents the technical details of the revised MSC
method; Section \ref{related_work} discusses the related work; 
Section \ref{evaluation} presents an empirical evaluation on
our proposed method; and finally, Section \ref{con} concludes our
work.

\section{Preliminaries}
\label{preliminaries}
The technique proposed in this paper is for description logic
$\mathcal{SHI}$. For technical reasons, we need a \emph{constrained}
use of nominals on certain conditions (i.e. assertion cycles), which requires logic
$\mathcal{SHIO}$. Thus, in this section, we give a brief
introduction to formal syntax and semantics of logic $\mathcal{SHIO}$,
DL ontologies, and basic reasoning tasks for derivation of logical
entailments from DL ontologies.

\subsection{Description Logic $\mathcal{SHIO}$}

The vocabulary of description logic $\mathcal{SHIO}$ includes a set $\bf
R$ of named roles with a subset $\bf R_+ \subseteq R$ of transitive
roles, a set  $\bf C$ of named (atomic) concepts, and a set $\bf I$ of
named individuals.  

\begin{definition} [{\bf $\mathcal{SHIO}$-Role}]
A role in $\mathcal{SHIO}$ is either a named (atomic) one $R \in \bf
R$ or an inverse one $R^-$ with $R \in \bf R$, and the complete role
set of  $\mathcal{SHIO}$ can be denoted ${\bf R}^* = {\bf R} \cup
\{R^-\ |\ R \in {\bf R}\}$.  To avoid role representation such
as $R^{--}$, a function \emph{Inv($\cdot$)} is defined, such that
\emph{Inv}($R$) = $R^-$ if $R$ is a role name, and \emph{Inv}($R$)=$P$
if $R=P^-$ for some role name $P$.

A role $R$ is \emph{transitive}, denoted \emph{Trans}($R$), if either
$R$ or \emph{Inv}($R$) belongs to $\bf R_+$.
\end{definition}

\begin{definition} [{\bf $\mathcal{SHIO}$-Concept}]
A $\cal SHIO$-concept is either an atomic (named) concept or a complex one
that can be defined using the following constructs recursively
\begin{equation*} 
\begin{split}
	C,\ D \quad  ::= \quad & A \ |\ \{o\}\ |\ \top\ |\ \bot\ |\
        \neg C\ |\ C \sqcap D\ |\ C \sqcup D\ |  \\
	&  \forall R.C\ |\ \exists R.C 
\end{split}
\end{equation*}
where $A$ is an atomic concept in $\bf C$, $o$ is a named individual, and $R \in {\bf R}^*$.  
\end{definition}

Description logic $\mathcal{SHI}$ is then defined as a fragment of
$\mathcal{SHIO}$, which disallows the use of nominal (i.e. $\{o\}$) 
as a construct for building complex concepts.

\begin{definition} [{\bf $\mathcal{SHIO}$ Semantics}]
The meaning of an entity in $\mathcal{SHIO}$ is defined by a
model-theoretical  semantics using an \emph{interpretation} denoted 
${\mathcal{I}}=(\Delta^{\mathcal{I}},.^{\mathcal{I}})$, where
$\Delta^{\mathcal{I}}$ is referred to as a non-empty domain and
$.^{\mathcal{I}}$ is an interpretation function. 

The function $.^{\mathcal{I}}$ maps every atomic concept in $\bf C$ to
a subset of $\Delta^{\mathcal{I}}$, every ABox individual to an element of
$\Delta^{\mathcal{I}}$, and every role to a subset of
$\Delta^{\mathcal{I}}\ \times\ \Delta^{\mathcal{I}}$. Interpretation
for other concepts and inverse role are given below: 
\renewcommand{\arraystretch}{1.1}
\begin{center}
\begin{tabular}{rcl}
$\top^{\mathcal{I}}$ & = & $\Delta^{\mathcal{I}}$ \\
$\bot^{\mathcal{I}}$ & = & $\emptyset$ \\
$\{o\}^{\mathcal{I}}$ & = & $\{o^{\mathcal{I}}\}$ \\
$\neg C^{\mathcal{I}}$ & = & $\Delta^{\mathcal{I}} \backslash C^{\mathcal{I}}$ \\
$(R^-)^{\mathcal{I}}$ & = & $\{(y,x)\ |\ (x,y) \in R^{\mathcal{I}}\}$ \\
$(C \sqcap D)^{\mathcal{I}}$ & = &  $C^{\mathcal{I}} \cap D^{\mathcal{I}}$ \\
$(C \sqcup D)^{\mathcal{I}}$ & = &  $C^{\mathcal{I}} \cup D^{\mathcal{I}}$ \\
$(\exists R.C)^{\mathcal{I}}$ & = & $\{x\ |\ \exists y.(x,y) \in R^{\mathcal{I}} \wedge y \in C^{\mathcal{I}} \}$ \\
$(\forall R.C)^{\mathcal{I}}$ & = & $\{x\ |\ \forall y.(x,y) \in R^{\mathcal{I}} \wedge y \in C^{\mathcal{I}} \}$ \\
\end{tabular}
\end{center}
\end{definition}

\begin{definition} [{\bf Simple-Form Concept}]
A concept is said to be in \emph{simple form}, if the maximum level of
nested quantifiers in this concept is less than 2. 
\end{definition}
\noindent
For example, given an atomic concept $A$, both $A$ and $\exists R.A$
are simple-form concepts, while $\exists R_1.(A \sqcap \exists R_2.A)$
is not, since its maximum level of nested quantifiers is two. Notice
however, an arbitrary concept can be \emph{linearly} reduced to the simple 
form by assigning new concept names for fillers of the quantifiers.
For example, $\exists R_1.\exists R_2.C$ can be converted to $\exists
R_1.D$ by letting $D \equiv \exists R_2.C$ where $D$ is a new concept
name. 

\subsection{DL Ontologies and Reasoning}

\begin{definition} [{\bf $\mathcal{SHI}$-Ontology}]
A $\mathcal{SHI}$ ontology is a tuple, denoted $\mathcal{K = (T,A)}$,
where $\mathcal{T}$ is called a \emph{TBox} and $\mathcal{A}$ is called
an \emph{ABox}. 

The TBox $\mathcal{T}$ is constituted by a finite set of role
inclusion axioms (i.e. $R_1 \sqsubseteq R_2$ with $R_1,R_2  \in \bf
R^*$) and a finite set of concept inclusion axioms in the form of $C
\sqsubseteq D$ and $C \equiv D$, where $C$, $D$ are $\mathcal{SHI}$
concepts. The former is called a \emph{general concept inclusion axiom (GCI)},
and the latter can be simply converted into two GCIs as $C \sqsubseteq
D$ and $D \sqsubseteq C$. 

The ABox $\mathcal{A}$ consists of a finite set of assertions, in the
form of $A(a)$ (concept assertion) and $R(a,b)$ (role assertion), where
$A$ is a concept, $R$ is a role, and $a,b$ are named individuals in $\bf I$.
\end{definition}

In a role assertion $R(a,b)$, individual $a$ is referred to as a \emph{R-predecessor}
of $b$, and $b$ is a \emph{R-successor} (or  \emph{R$^-$-predecessor}) of $a$.
If $b$ is a $R$-successor of $a$,  $b$ is also called a \emph{R-neighbor} of $a$.

An interpretation $\mathcal{I}$ satisfies an axiom $C \sqsubseteq D$
(written $\mathcal{I} \models C \sqsubseteq D$), iff $C^\mathcal{I}
\subseteq D^\mathcal{I}$, and $\mathcal{I}$ satisfies an axiom or assertion: 
\renewcommand{\arraystretch}{1.1}
\begin{center}
\begin{tabular}{rcl}
$R_1 \sqsubseteq R_2$ & $\it iff$ & $R_1^{\mathcal{I}} \subseteq R_2^{\mathcal{I}}$ \\
$C(a)$ & $\it iff$ & $a^{\mathcal{I}} \in C^{\mathcal{I}}$ \\
$R(a,b)$ & $\it iff$ & $(a^{\mathcal{I}}, b^{\mathcal{I}}) \in R^{\mathcal{I}}$.
\end{tabular}
\end{center}

\noindent
If $\mathcal{I}$ satisfies every axiom and assertion of an ontology
$\mathcal{K}$, $\mathcal{I}$ is called a \emph{model} of
$\mathcal{K}$, written $\mathcal{I} \models \mathcal{K}$. In turn,
$\mathcal{K}$ is said  \emph{satisfiable} iff it has at least one
model; otherwise, it is \emph{unsatisfiable} or \emph{inconsistent}. 

\begin{definition} [{\bf Logical Entailment}]
Given an ontology $\mathcal{K}$ and an axiom $\alpha$, $\alpha$ is called a
logical entailment of $\mathcal{K}$, denoted $\mathcal{K} \models \alpha$, 
if $\alpha$ is satisfied in every model of $\mathcal{K}$.
\end{definition}

\begin{definition} [{\bf Instance checking}]
Given an ontology $\mathcal{K}$, a DL concept $C$ and an
individual $a \in \bf I$, instance checking is defined to test
if ${\mathcal{K}} \models C(a)$ holds.
\end{definition}

Notice that, instance checking is considered the central reasoning
service for information retrieval from ontology ABoxes \cite{Schaerf1994},
and more complex reasoning services, such as instance retrieval, can
be realized based on this basic service. Instance checking can 
also be viewed as a procedure of individual ``\emph{classification}''
that verifies if an individual can be classified into some defined 
DL concepts. 

An intuition to implement this instance checking service is to convert
it into a concept subsumption test by using the so-called
\emph{most specific concept} (MSC) method. 

\begin{definition} [{\bf Most Specific Concept} \cite{Donini1992}]
Let $\mathcal{K=(T,A)}$ be an ontology, and $a$ be an individual in
$\bf I$. A concept $C$ is called the \emph{most specific concept} for $a$
w.r.t. $\mathcal{A}$, written ${\rm MSC}(\mathcal{A}, a)$, if for  %
every concept $D$ that $\mathcal{K} \models D(a)$, $\mathcal{T}
\models C \sqsubseteq D$. 
\end{definition}

The MSC method turns the instance checking into a TBox reasoning problem. 
That is, once the most specific concept ${\rm MSC}(\mathcal{A}, a)$ of
an individual $a$ is known, to decide if $\mathcal{K} \models D(a)$
holds for an arbitrary concept $D$, it suffices to test if
$\mathcal{T} \models {\rm MSC}(\mathcal{A}, a) \sqsubseteq D$
\cite{Donini2007}. 

Ontology reasoning algorithm in current systems (e.g. Pellet, and HermiT, etc.) 
are based on (hyper) tableau algorithms \cite{Haarslev2001,Motik2007,Sirin2007,Tsarkov2006}.
For details of a standard tableau algorithm for $\cal SHIO$, we refer
readers to the work in \cite{Horrocks2007b}.

\subsubsection{Assumption}
For accuracy of the technique presented in this paper,
without loss of generality, we assume all ontology concepts are in
simple form as defined previously, and the concept in any concept
assertion is atomic.

\section{Classification of Individuals}
\label{msc_method}

\subsection{The MSC method}
The MSC method for individual checking is based on the idea that,
an individual can be classified into a given concept $D$, 
if and only if there exists a concept behind its ABox assertions 
subsumed by $D$ \cite{Donini1992,Donini1994,Nebel1990}. 
Computation of the MSC for a given individual then demands converting 
its ABox assertions into a concept. This task can be easily accomplished
if the individual possesses only concept assertions, by simply collapsing the
involved concepts into a single term using the concept conjunction. 
When role assertions are involved, however, a more complex procedure
is demanded, and the method we used here is called \emph{rolling-up}
\cite{Horrocks2000b}, which is elaborated in the next section. 

Using the MSC method for instance checking might eliminate the memory
limitation for reasoning with large ABoxes, especially when the ABox
$\mathcal{A}$  consists of data in great diversity and isolation.
This is simply because each computed ${\rm MSC}(\mathcal{A}, a)$ 
should comprise of only related information of the given individual,
and makes the subsumption test (i.e. $\mathcal{K} \models {\rm
  MSC}(\mathcal{A}, a) \sqsubseteq D$) as efficient as an ontology
reasoning that explores only a (small) portion of $\mathcal{A}$.  

However, as discussed in Section \ref{intro}, due to the support of
existential restrictions in DLs, great complexity for computation of
MSC's may arise when role assertions are involved. Besides, due to the
\emph{completeness} that should be guaranteed by each
MSC$(\mathcal{A}, a)$ (i.e. the MSC should be subsumed by any concept
that the individual $a$ belongs to.), the resulting MSC's may turn out
to be a very large concept, whenever there is a great number of
individuals in $\mathcal{A}$ connected to each other by role
assertions. In the \emph{worst} case, reasoning with a MSC may
degenerate into a complete ABox reasoning that could be prohibitively
expensive.  For example, when ${\rm MSC}(\mathcal{A}, a)$ preserves
complete information of $\mathcal{A}$, its interpretation will form a
tableau, the size of which can be in the same scale of $\mathcal{A}$.

\subsection{The Call-by-Need Strategy}

Since the larger than desired sizes of MSC's are usually caused by its completeness as
discussed above, a possible optimization to the MSC method is thus to abandon the
completeness that is required to deal with any query concepts,
and to apply a \emph{``call-by-need''} strategy. That is, for an
arbitrary query concept $D$, instead of computing the 
  MSC for each individual $a$, we compute a
concept that is only \emph{specific-enough}  to determine if $a$ 
can be classified into $D$. As suggested by its name, this
revision to the original MSC method, instead of taking the complete
information of individual $a$ when computing the "MSC",  
will consider only the ABox assertions that are relevant to the current query
concept.   

A simple way to realize this strategy is to assign a fresh name $A$
every time to a given (complex) query concept $D$ by adding the axiom 
$A \equiv D$ to $\mathcal{T}$,\footnote{Note that, to follow the
simple-form concept restriction, multiple axioms may be added.}
and to concentrate only on ABox assertions that would (probably)
classify an individual $a$ into $A$ w.r.t.$\mathcal{T}$.
Consequently, this implementation requires an analysis of the ontology
axioms/assertions, such that the possibility of each role assertion to
affect individual classification (w.r.t.  named concept in
$\mathcal{T}$) can be figured out.  Computation of a specific-enough
concept should then concentrate on role assertions that are
\emph{not impossible}. We abuse the notation here to denote this
specific-enough concept for individual $a$ w.r.t. ABox
$\mathcal{A}$, current query concept $\mathcal{Q}$, and named concepts
in $\mathcal{T}$ as ${\rm MSC}_{\mathcal{T}}(\mathcal{A,Q},a)$,  
and we call the method that uses MSC$_\mathcal{T}$ for instance
checking the MSC$_\mathcal{T}$ method.
\begin{definition}
Let $\mathcal{K=(T,A)}$ be an ontology, $a$ be an individual in
$\mathcal{A}$, and $\mathcal{Q}$ a current query concept for
individuals. A concept $C$ is called a specific-enough concept for
$a$ w.r.t. named concepts in $\mathcal{T}$, $\mathcal{Q}$ and
$\mathcal{A}$, written ${\rm MSC}_{\mathcal{T}}(\mathcal{A, Q}, a)$,
if $\mathcal{K} \models \mathcal{Q}(a)$, $\mathcal{T} \models C
\sqsubseteq \mathcal{Q}$. 
\end{definition}

\noindent
Since in our procedure we will add the query concept $\mathcal{Q}$ into
$\mathcal{T}$ as a named concept, we can simplify the notation 
 ${\rm MSC}_{\mathcal{T}}(\mathcal{A, Q}, a)$ as ${\rm MSC}_{\mathcal{T}}(\mathcal{A}, a)$.

\subsection{A Syntactic Premise}

To decide whether a role assertion could affect classification of a
given individual, a sufficient and necessary condition as stated
previously is that, the concept behind this assertion conjuncted
with other essential information of the individual should be subsumed
by the given concept w.r.t. $\mathcal{T}$
\cite{Donini1992,Donini1994,Nebel1990}.  Formally, for a role
assertion $R(a,b)$ that makes individual $a$ classified into a concept
$A$,  the above sufficient and necessary condition in $\mathcal{SHI}$ can be
expressed as:
\begin{equation} \label{sem_cond}
\mathcal{K} \models \exists R.B \sqcap A_0 \sqsubseteq A,
\end{equation}
where $b \in B$ is entailed by $\mathcal{K}$, and concept $A_0$ summarizes the rest of the
information of $a$ that is also essential to this classification,
with $A_0 \not\sqsubseteq A$.

As shown in \cite{Xu2013}, for subsumption (\ref{sem_cond}) to hold when
$A$ is a named concept, there must exist some role restriction
$\exists R'.C$ with $R \sqsubseteq R'$ in left-hand side of TBox
axioms (see (\ref{syn-cond}) and the following axiom equivalency) for concept
definition; otherwise $\exists R.B$ is not comparable
(w.r.t. subsumption) with other concepts (except $\top$ and its equivalents).
This syntactic condition for the deduction of (\ref{sem_cond}) is formally
expressed in the following proposition.

\begin{proposition}[\cite{Xu2013}] \label{opt-p1} 
Let $\mathcal{K=(T,A)}$ be a $\mathcal{SHI}$ ontology with simple-form
concepts only, $\exists R.B$, $A_0$ and $A$ be $\mathcal{SHI}$ concepts,
where $A$ is named. If 
\begin{center}
$\mathcal{K} \models \exists R.B \sqcap A_0 \sqsubseteq A$
\end{center}
with $A_0 \not\sqsubseteq A$, there must exist some GCIs in $\mathcal{T}$
in the form of:
\begin{equation} \label{syn-cond}
\begin{split}
	& \exists R'.C_1 \bowtie C_2 \sqsubseteq C_3 
	\end{split}
\end{equation}
where $R \sqsubseteq R'$ and $\bowtie$ is a place holder for $\sqcup$
and $\sqcap$, $C_i$'s are $\mathcal{SHI}$ concepts. Also note the
following equivalence: 
\begin{center}
\begin{tabular}{rcl}
$\exists R.C \sqsubseteq D$ & $\equiv$ & $\neg D \sqsubseteq \forall R.\neg C$   \\
$\exists R.C \sqsubseteq D$ & $\equiv$ & $C \sqsubseteq \forall R^-.D$\\
$C_1 \sqcap C_2 \sqsubseteq D$ & $\equiv$ & $C_1 \sqsubseteq D \sqcup \neg C_2$
\end{tabular}
\end{center}
\end{proposition}

This proposition is proven in \cite{Xu2013}. It states in fact a
syntactic premise in $\mathcal{SHI}$ for a role assertion to be
essential for some individual classification. That is, if a role
assertion $R(a,b)$ is essential for derivation of $A(a)$ for some
named concept $A$, there must exist a related axiom in $\mathcal{T}$ in
the form of (\ref{syn-cond}) for $R \sqsubseteq R'$. We denote this
syntactic premise for $R(a,b)$ to affect $a$'s classification as $\tt
SYN\_COND$.  Using this condition, we can easily rule out role assertions
that are definitely irrelevant to the query concept and will not be
considered during the computation of a $\rm MSC_{\cal T}$.

\section{Computation of MSC$_\mathcal{T}$}
\label{method}

In this section, we present the technique that computes a
MSC$_\mathcal{T}$ for a given individual w.r.t. a
given query.  
We assume the ABox considered here is consistent, since for any
inconsistent ABox, the MSC$_\mathcal{T}$ is always the bottom concept 
$\bot$ \cite{Baader1998}. Essentially, the task is to convert ABox
assertions into a single concept for a given individual, using
the concept conjunction and the so-called \emph{rolling-up}
technique. This rolling-up technique was introduced in
\cite{Horrocks2000b} to convert conjunctive queries into concept
terms, and was also used by \cite{Krotzsch2008}
to transform datalog rules into DL axioms. We adapt this technique 
here to roll up ABox assertions into DL concepts.

\subsection{The Rolling-up Procedure}
\label{roll-up}
Converting concept assertions into a concept is straightforward
by simply taking conjunction of the involved concepts.
When role assertions are involved, the rolling-up technique can be
used to transform assertions into a concept by eliminating
individuals in them. For example, given the following assertions 
\begin{equation} \label{ex1}
\tt Male(Tom), \ hasParent(Tom,Mary), \ Lawyer(Mary),
\end{equation}
transforming them for individual $\tt Tom$ using the
rolling up and concept conjunction can generate a single concept assertion
\begin{equation*} 
\tt (Male \sqcap \exists hasParent.Lawyer) (Tom).
\end{equation*}

\noindent
{\bf Generalize the Information}: The transformation here is for
individual $\tt Tom$, and if individual
$\tt Mary$ is not explicitly indicated in the query, it should
be sufficient to rewrite $\tt hasParent(Tom,Mary),\ Lawyer(Mary)$ into  
$\tt \exists hasParent.Lawyer (Tom)$, without loss of any information
that is essential for query answering. Even if $\tt Mary$ is explicitly
indicated in the query, we can still eliminate it by using a \emph{representative}
concept that stands for this particular individual in the given ABox
\cite{Horrocks2000}. For example, we can add an assertion
$A_{mary}({\tt Mary})$ to the ABox, where $A_{mary}$ is a new concept
name and a representative concept for individual $\tt Mary$. The
above role assertions for $\tt Tom$ then can be transformed into concept $\tt \exists
hasParent.(Lawyer \sqcap {\it A_{mary}}) (Tom)$; and if the query is
also rewritten using concept $A_{mary}$, the \emph{completeness}
of the query answering can be guaranteed, as indicated by the 
following theorem \cite{Horrocks2000}.

\begin{theorem} [\cite{Horrocks2000}]
Let $\mathcal{K=(T,A)}$ be a DL ontology, $a,b$ be two individuals in $\mathcal{A}$,
$R$ a role, and $C_1, \cdots, C_n$ DL concepts. Given a
representative concept name $A_b$ not occurring in $\mathcal{K}$:
\begin{center}
$\mathcal{K} \models R(a,b) \wedge C_1(b) \wedge \cdots \wedge C_n(b)$
\end{center}
if and only if 
\begin{center}
$\mathcal{K} \cup \{A_b(b)\} \models \exists R.(A_b \sqcap C_1 \sqcap
\cdots \sqcap C_n) (a)$
\end{center}
\end{theorem}

The rolling-up procedure here can be better understood by considering a
\emph{graph} induced by the role assertions to be rolled up, which is
defined as follows: 

\begin{definition}
A set of ABox role assertions in $\mathcal{SHI}$ can be represented by
a graph $\mathcal{G}$, in which there is a node $x$ for each
individual $x$ in the assertions, and an edge between node $x$ and $y$ 
for each role assertion $R(x,y)$.
\end{definition}

Notice that, due to the support of inverse roles in $\mathcal{SHI}$,
edges in $\mathcal{G}$ are not directed. A \emph{role path} in the
graph $\mathcal{G}$ is then defined as a set of roles corresponding to
the set of edges (no duplicate allowed) leading from one node to another.
For example,  given assertions $R_1(x,y)$ and $R_2(z,y)$,
the role path from $x$ to $z$ is $\{R_1, R_2^-\}$, and its reverse
from $z$ to $x$ is $\{ R_2, R_1^-\}$.

The rolling-up for a given individual $x$ is then able to generate
concepts by eliminating individuals in branches of the
\emph{tree-shaped} graph $\mathcal{G}$, starting from the leaf nodes
and rolling up towards the root node indicated by $x$. Moreover, all the
information of each individual being rolled up should be absorbed into
a single concept by conjunction during the procedure. For
example, if we have additional assertions 
\begin{center}
$\tt hasSister(Mary, Ana)$ and $\tt Professor(Ana)$
\end{center}
for $\tt Mary$ in (\ref{ex1}),  the rolling-up for $\tt Tom$ should then generate concept
\begin{center}
$\tt Male \sqcap \exists hasParent.(Lawyer \sqcap \exists hasSister.Professor)$.
\end{center}

\noindent
{\bf Inverse Role}: The support of inverse roles in $\cal SHI$ makes this rolling-up procedure
bidirectional, thus, making it
applicable to computing MSC$_\mathcal{T}$ for any individual in the
ABox. For example, to compute a MSC$_\mathcal{T}$ for individual $\tt
Mary$ in example (\ref{ex1}), we simply treat this individual as the root,
and roll up assertions from leaves to root to generate the 
concept 
\begin{center}
$\tt Lawyer \sqcap \exists hasParent^-. Male$.
\end{center}

\noindent
{\bf Transitive Role}: In the rolling-up procedure, no particular care
needs to be taken to deal with transitive roles, since any role
assertions derived from transitive roles will be automatically
preserved \cite{Horrocks2000}. For example, given $R$ a transitive
role, $R(a,b)$, $R(b,c)$ two role assertions, and $B(b)$,
$C(c)$ two concept assertions in the ABox, rolling-up these four assertions for
individual $a$ can generate assertion $\exists R.(B \sqcap \exists R.C)(a)$,
from which together with the TBox axioms, we can still derive the fact that
\begin{center}
$(\exists R.(B \sqcap \exists R.C) \sqcap  \exists R.C)(a)$.
\end{center}

\noindent
{\bf Assertion Cycles}: This rolling-up technique, however, may suffer
information loss if the graph $\mathcal{G}$ contains cycles (i.e. a
role path leading from one node to itself without duplicate graph edges). 
For example, given the following two
assertions:  
\begin{equation} \label{circle-example}
R_1(x,y) \quad\ {\rm and}\ \quad R_2(x,y),
\end{equation}
individuals $x$ and $y$ are related by two roles, and a cycle is thus
induced in the corresponding graph. Rolling-up assertions for
individual $x$ using the method described above might generate concept $\exists
R_1.\top \sqcap \exists R_2.\top$, and the fact that $x$ is connected
to the same individual $y$ through different roles is lost. Consequently,  
this may compromise the resulting concept as a specific-enough 
concept for $x$ to answer the current query. For example, let $C$ be a query
concept defined as:  
\begin{center}
 $\exists R_1.\exists R_2^-.\exists R_1.\top$.
\end{center}
It can be found out through ABox reasoning that individual $x$ is
an instance of $C$; while on the other hand, it is also not difficult
to figure out that $\exists R_1.\top \sqcap \exists R_2.\top$ is not
subsumed by $C$. 

Multiple solutions to this problem have been proposed, such as an
approximation developed by \cite{Kusters2001}, and the use of cyclic
concept definition with greatest fixpoint semantics \cite{Baader1998,Baader2003}. 
In this paper, we choose to use the \emph{nominal} (e.g. $\{x\}$) to handle circles 
as suggested by \cite{Donini1992,Schaerf1994}, which allows explicit
indication of named individuals in a concept, hence, being able to 
indicate the joint node of a cycle. The above two assertions
in (\ref{circle-example}) then can be transformed into a concept
for individual $x$ as either $\exists R_1.\{y\} \sqcap \exists
R_2.\{y\}$ or $\{x\}  \sqcap \exists R_1.\exists R_2^-.\{x\}$,
each with the nominal used for a chosen joint node, and both preserve
complete information of the cycle. 
In our approach, when rolling up a cycle in $\mathcal{G}$, we
always treat the cycle as a single branch and generate concepts of the
second style. That is, our procedure will treat a chosen joint node $x$ as 
both the tail (i.e. leaf) and the head (i.e. root) of the branch. For clarity
of the following presentation, we denote the tail as $x^t$ and the
head as $x^h$.   

Based on the discussion so far, the transformation of assertions for a
given individual now can be formalized as follows. Let $x$ be a named
individual,  and $\gamma$ be an ABox assertion for $x$. 
$\gamma$ can be transformed into a concept $C_\gamma$ for $x$:  
\begin{equation*}
C_\gamma = \left\{ \begin{array}{lrl}
             C & \mbox{if} & \gamma=C(x)\ |\ C(x^h) \\
             \exists R.D & \mbox{if} & \gamma=R(x,y),\ {\rm and}\ D(y) \\
             \exists R^-.D & \mbox{if} & \gamma=R(y,x),\ {\rm and}\ D(y) \\
          \{x^h\} \sqcap \exists R.\{x^t\} & \mbox{if} & \gamma=R(x^h,x^t) \\
          \{x^h\} \sqcap \exists R.D & \mbox{if} & \gamma=R(x^h,y),\ {\rm and}\  D(y) \\
          \{x^h\} \sqcap \exists R^-.D & \mbox{if} & \gamma=R(y,x^h),\ {\rm and}\  D(y) \\
          \{x^t\} & \mbox{if} & \gamma=R(x^t,y) | R(y, x^t) | C(x^t) \\
          \end{array}\right.
\end{equation*}

\noindent
Notice that, concept $D$ here is a obtained concept when rolling
up branch(es) in $\mathcal{G}$ up to node $y$, and transforming
any assertion of a cycle tail $x^t$ always generates $\{x^t\}$, as complete
information of $x$ will be preserved when rolling up to the head $x^h$.
Thereafter, given a set $S$ of all assertions of individual $x$,
${\rm MSC}_{\mathcal{T}}(\mathcal{A},x)$ can be obtained by rolling up
all branches induced by role assertions in $S$ and taking the
conjunction of all obtained concepts. When $S$ is empty, however,
individual $x$ in the ABox can only be interpreted as an element of
the entire domain, and thus, the resulting concept is simply the top
entity $\top$. Computation of a ${\rm
  MSC}_{\mathcal{T}}(\mathcal{A},x)$ then can be formalized using the 
following equation: 
\begin{equation*}
{\rm MSC}_{\mathcal{T}}(\mathcal{A},x) = \left\{ \begin{array}{lrl}
               \top & \mbox{if} & S = \emptyset \\
             \sqcap_{\gamma \in S} \ C_\gamma &
               \mbox{if} & {\rm otherwise} \\
               \end{array}\right.
\end{equation*}

\subsection{Branch Pruning}
\label{prun}

To apply the call-by-need strategy, the previously defined
syntactic premise $\tt SYN\_COND$ is employed, and a branch to
be rolled up in graph $\mathcal{G}$ will be truncated at the point
where the edge does not have $\tt SYN\_COND$ satisfied.
More precisely, if an assertion $R(x,y)$ in a branch does not have
the corresponding $\tt SYN\_COND$ satisfied, it will not affect any
classification of individual $x$ w.r.t. $\mathcal{T}$. Moreover, any
effects of the following assertions down the branch will not be able
to propagate through $R(x,y)$ to $x$,  and thus should not be
considered during the rolling-up of this branch. 

This branch pruning technique could be a simple yet an efficient way to
reduce complexity of a $\rm MSC_\mathcal{T}$, especially for those
practical ontologies, where many of the ABox individuals may have a huge
number of role assertions and only a (small) portion of them have
$\tt SYN\_COND$ satisfied. For a simple example, consider an 
individual $x$ in an ontology ABox with the following assertions:
\begin{center}
$R_1(x,y_1), R_2(x,y_2), \cdots, R_n(x,y_n)$,
\end{center}
where $n$ could be a very large number and only $R_1$ has $\tt
SYN\_COND$ satisfied. Rolling up these assertions
for individual $x$ without the pruning will generate the concept 
\begin{center}
$\exists R_1.C_1 \sqcap \exists R_2.C_2 \sqcap \cdots \sqcap \exists.R_n.C_n$,
\end{center}
where $y_i \in C_i$.
Using this concept for any instance checking of $x$ could be expensive, 
as its interpretation might completely restore the
tableau structure that is induced by these assertions. However, when
the pruning is applied, the new $\rm MSC_\mathcal{T}$ should be
$\exists R_1.C_1$, the only role restriction that is possible to affect
individual classification of $x$ w.r.t. named concepts in $\mathcal{T}$.

Going beyond such simple ontologies, this optimization technique may
also work in complex ontologies, where most of the role assertions
in ABox could have $\tt SYN\_COND$ satisfied. For example, consider
the following assertions 
\begin{center}
$R_0(x_0, x_1), R_1(x_1, x_2),\cdots, R_n(x_n, x_{n+1})$,
\end{center}
with all roles except $R_2$ having $\tt SYN\_COND$ satisfied.
Rolling up these assertions for individual $x_0$ will start from the
leaf $x_{n+1}$ up towards the root  $x_0$, and generate the concept
\begin{center}
$\exists R_0.\exists R_1.\cdots .\exists R_n.C$,
\end{center}
where $x_{n+1} \in C$.
However, with pruning applied, the rolling-up in this branch will start
from $x_2$ instead of $x_{n+1}$, since $R_2(x_2, x_3)$ will not affect
classification of individual $x_2$ w.r.t. $\mathcal{T}$ and the branch
is truncated at this point.

Furthermore, with branch pruning, cycles should only be considered in
the truncated graph, which may further simplify the computation of
MSC$_\mathcal{T}$'s.

\subsection{Further Optimization and Implementation}
\label{opt}

The branch pruning here is based on $\tt SYN\_COND$ to rule out
irrelevant assertions, which in fact can be further improved by developing
a more rigorous premise for a role assertion to affect individual classification.
For exposition, consider the following ontology $\mathcal{K}$:
\begin{equation} \label{example2}
(\{\exists R.C \sqsubseteq D\},\quad \{\neg D(a),\ R(a,b), \ \neg C(b)\}).
\end{equation}
When computing MSC$_\mathcal{T}(\mathcal{A},a)$ using the
proposed method, assertion $R(a,b)$ will be rolled up as the
corresponding $\tt SYN\_COND$ is satisfied. However, it is not 
difficult to see that, $R(a,b)$ here actually makes no contribution to
$a$'s classification, since individual $b$ is in the complement of
concept $C$, making $a$ an instance of $\exists R.\neg C$. Besides,
individual $a$ has already been asserted as an instance of concept
$\neg D$, and hence cannot be classified into $D$ unless the ABox is
inconsistent. 

With these observations, a more rigorous premise based 
on $\tt SYN\_COND$ can be derived. That is, to determine the
possibility for $R(a,b)$ to affect classification of individual $a$,
beyond checking  in $\mathcal{T}$  the existence of any axiom in the
form of
\begin{center}
 $\exists R'.C_1 \bowtie C_2 \sqsubseteq C_3$,
\end{center}
with $R \sqsubseteq R'$ and $\bowtie$ a place holder for $\sqcup$
and $\sqcap$, we also check the following cases for any found axiom:
\begin{LaTeXdescription}
\item[case 1.] if there is any concept $B_0$ in explicit concept assertions
  of individual $b$, such that  $\mathcal{K} \models B_0 \sqsubseteq
  \neg C_1$, or

\item[case 2.] if there is any concept $A_0$ in explicit concept
  assertions of individual $a$, such that $\mathcal{K} \models A_0
  \sqsubseteq \neg (C_3 \sqcup \neg C_2)$\footnote{Note the
   axiom equivalence $C_1 \sqcap C_2 \sqsubseteq D \equiv C_1 \sqsubseteq
    D \sqcup \neg C_2$.} or $\mathcal{K} \models A_0
  \sqsubseteq \neg C_3$, respectively for $\bowtie$ standing for
  $\sqcap$ or $\sqcup$. 
\end{LaTeXdescription}
If either one of the above cases happens, that particular $\exists
R'.C_1$ in the left hand side of the axiom in fact
makes no contribution to the inference of $a$'s classification, unless
the ABox is \emph{inconsistent} where MSC's are always $\bot$ \cite{Baader1998}.
Thus, a revised condition requires not only the existence of a related
axiom in the form of (\ref{syn-cond}) but also with none of the above cases
happening. We denote this condition as $\tt SYN\_COND^*$, and use it
to rule out assertions that are irrelevant to the current query. 

This optimization is useful to prevent rolling-up of
role assertions in an arbitrary direction on existence of related
axioms in $\mathcal{T}$. Instead, it limits the procedure to the
direction that is desired by the original intention underlying
the design of the given ontology. For example, in (\ref{example2}),
the axiom 
$\exists R.C \sqsubseteq D$
specifies that, any individual having a $R$-neighbor in
$C$ is an instance of $D$ and any individual having a
$R^-$-neighbor in $\neg D$ is an instance of $\neg C$.\footnote{Note
that $\exists R.C \sqsubseteq D$ is equivalent with $\exists R^-.\neg 
D \sqsubseteq \neg C$. } However, if individual $x$ is asserted to have a
$R$-neighbor in $\neg C$ or a $R^-$-neighbor in $D$, that role assertion
should not be rolled-up for $x$ just on existence of this axiom.

With all the insights discussed so far, an algorithm for computation
of MSC$_\mathcal{T}(\mathcal{A},x)$ is presented here as a \emph{recursive}
procedure, the steps of which are summarized in Figure \ref{opt-algo}.

\begin{figure}[t]
\centering
\resizebox{0.48\textwidth}{!}{
\begin{tabular}{p{0.48\textwidth}}
\hline 
\begin{enumerate}

\item In this recursive procedure, if $x$ has already been visited before
  (cycle detected), mark $x$ as the joint node and return $\{x^t\}$. 

\item Obtain a set $S$ of all explicit assertions in $\mathcal{A}$ of
  the given individual $x$, which have not been visited before.

\item For every role assertion $\gamma:R_i(x,y_i) \in S$ (respectively
  $R_i(y_i,x)$) that has $\tt SYN\_COND^*$ satisfied and has not been visited yet, 
invoke this procedure \emph{recursively} to compute concept $D_i$ for $y_i$. 
  The rolling-up in this branch for $x$ then yields 
   $C_\gamma = \exists R_i.D_i$ (respectively $\exists R_i^-.D_i$).

\item For every concept assertion $\gamma:C_i(x) \in S$,
  $C_\gamma = C_i$.  

\item Return MSC$_{\mathcal{T}}(\mathcal{A},x)$ that equals to:
\begin{equation*}
 \left\{ \begin{array}{lrl}
               \top & \mbox{if} & S = \emptyset \\
               \sqcap_{\gamma \in S} \ C_\gamma \sqcap \{x^h\} &
               \mbox{if} & S\not= \emptyset,  {\rm\ and\ } x {\rm\ is\ marked } \\
               \sqcap_{\gamma \in S} \ C_\gamma & \mbox{if} & {\rm otherwise } \\
               \end{array}\right.
\end{equation*}

\end{enumerate} \\ 
\hline
\end{tabular}}
\caption{A recursive procedure for computation of MSC$_\mathcal{T}(\mathcal{A},x)$.}
\label{opt-algo}
\end{figure}

\begin{proposition} [Algorithm Correctness]
The algorithm presented in Figure \ref{opt-algo} computes a
MSC$_\mathcal{T}(\mathcal{A},x)$ for a given $\mathcal{SHI}$
ontology $\mathcal{(T,A)}$ and an individual $x$ in $\mathcal{A}$.
\end{proposition}

\begin{proof}
We prove by induction.
\begin{LaTeXdescription}
\item[Basis:] For a leaf node $x$ in $\mathcal{G}$, which has no other 
  role assertions except those up the branches, rolling it up yields the 
  conjunction of concepts in its concept assertions, which preserves
  sufficient information of the part of the branch being rolled so far. 
  If $x$ is the tail of a cycle, returning $\{x^t\}$ is sufficient, as other
  information of $x$ will be gathered when the rolling-up comes
  to the head.

\item[Inductive Step:] Let $x$ be a node in the middle of some
  branch(es) in $\mathcal{G}$. For every role assertion $R_i(x,y_i)$
  of $x$ down the branch, assume the procedure generates a concept
  $D_i$ for rolling up to each node  $y_i$, which preserves sufficient
  information  (w.r.t. current query) of the part of branches being
  rolled up so far. Then, rolling up each $R(x,y_i)$ generates $\exists R.D_i$,
  and together with concept assertions of $x$, the concept conjunction
  preserves sufficient information of all branches being rolled up to $x$. 
  If $x$ is marked as a joint node of a cycle,  $\{x^h\}$ is also in the
  conjunction, so that the circular path property can be preserved.

  If $x$ is the root node, the conjunction is thus a
  MSC$_\mathcal{T}(\mathcal{A},x)$ that preserves sufficient
  information of $x$ w.r.t. current query.
\end{LaTeXdescription}

This algorithm visits every relevant ABox assertion at most once, and
it terminates after all related assertions are visited.  
\end{proof}

\section{Related Work}
\label{related_work}

The idea of most specific concept for instance checking was first
discussed in \cite{Nebel1990}, and later extensively studied by
\cite{Donini1992,Donini1994} for the algorithms and the computational
complexity. To deal with existential restrictions when computing the
most specific concept, \cite{Baader1998,Baader1999,Baader2003}
discussed the use of cyclic concepts with greatest fixpoint semantics
for preservation of information induced by the role assertions, and \cite{Kusters2001}
also proposed an approximation for most specific concept in DLs with
existential restrictions.

On the other hand, for efficient ABox reasoning and instance checking,
various optimization techniques have been developed, including
lazy unfolding, absorption, heuristic guided search, exploration
of Horn clauses of DLs \cite{Horrocks2007,Motik2007,Motik2009}, model
merging \cite{Horrocks1997} and extended pseudo model merging technique
\cite{Haarslev2001b,Haarslev2008}. 

A common direction of these optimization techniques is to reduce the
high degree of nondeterminism that is mainly introduced by GCIs in the
TBox: given an GCI $C \sqsubseteq D$, it can be converted to a
disjunction $C \sqcup \neg D$, for which a tableau algorithm will have 
to nondeterministically choose one of the disjuncts for tableau
expansion, resulting in an exponential-time behavior of the tableau algorithm
w.r.t. the data size. Absorption optimizations
\cite{Horrocks2007,Hudek2006,Tsarkov2004} were developed to 
reduce such nondeterminism by combining GCIs for unfoldable concepts,
such that the effectiveness of lazy unfolding can be maximized.
For example, axioms $A \sqsubseteq C$ and $A \sqsubseteq D$ can be
combined into $A \sqsubseteq C \sqcap D$, where $A$ is a named
concept; then the inference engine can deduce $C\sqcap D (a)$ if
the ABox contains $A(a)$.  Notice however, this
technique may allow only parts of TBox axioms to be absorbed, thus, may
not be able to eliminate all sources of nondeterminism especially when
ontologies are complex. Based on the absorption optimization,
\cite{Wu2012} proposed an approach for efficient ABox reasoning for
$\mathcal{ALCIQ}$ that will convert ABox assertions into TBox axioms,
apply a absorption technique on the TBox, and covert instance retrieval
into concept satisfaction problems.

Another way to reduce nondeterminism is the exploration of Horn clauses in DLs,
since there exist reasoning techniques for Horn clauses that can be deterministic
\cite{Grosof2003,Motik2009}. \cite{Motik2009} takes advantage of this
in their HermiT reasoner by preprocessing a DL ontology into
DL-clauses and invoking the hyperresolution for the Horn clauses,
avoiding unnecessary nondeterministic handling of Horn problems in
existing DL tableau calculi. 

For non-Horn DL, techniques such as model merging \cite{Horrocks1997}
and pseudo model merging \cite{Haarslev2001b} can be used to capture
some deterministic information of named individuals. These techniques
are based on the assumption of a consistent ABox and the observation that
usually individuals are members of a small number of concepts. The
(pseudo) model merging technique merges clash-free tableau models that are
constructed by disjunction rules for a consistent ABox, and can figure out
individuals that are obviously non-instance of a given concept. For
example, if in one tableau model individual $a$ belongs to concept $C$
while in another $a$ belongs to $\neg C$, it is then obvious that
individual $a$ cannot be deterministically inferred to be an instance
of concept $C$, thus, eliminating the unnecessary instance checking
for $C(a)$.

Another option for scalable ABox reasoning is the use of tractable DL
languages. For example, the description logic $\cal EL$ and 
its extension $\cal EL^{++}$, which allow existential restrictions
and conjunction as introduced by \cite{Baader2005,Baader2008}, possess
intriguing algorithmic properties such that the satisfiability problem
and implication in this DL language can be determined in polynomial
time. Another notable example of lightweight DLs is the so-called 
\emph{DL-LITE} family identified by \cite{Calvanese2005}, which is
specifically tailored to capture basic DL properties and expressivity
while still be able to achieve low computational complexity for both TBox
and ABox reasoning. In \cite{Calvanese2006,Calvanese2007} they further
identified that, for conjunctive queries that are FOL-reducible, answering
them in ontologies of any DL-LITE logic enjoys a LOGSPACE data complexity. 

Based on the above lightweight DLs, efficient DL reasoners are developed,
such as OWLIM \cite{Bishop2011}, ELK reasoner \cite{Kazakov2011}, and 
Oracle's native inference engine for RDF data sets \cite{Wu2008}.

\cite{Zhou2013} proposed an approximation technique for instance
retrieval, which computes both lower bound and 
upper bound of an answer set of individuals for a given query
concept. Their approach invokes an axiom rewriting procedure that
converts an ontology in Horn DL into a datalog program, and then 
uses Oracle's native inference engine to derive the bounds for query answering.

Recently, techniques for partitioning or modularizing ABoxes into
logically-independent fragments have been developed
\cite{Wandelt2012,Xu2013}. These techniques partition ABoxes into
logically-independent modules, such that each will preserve complete
information of a given set of individuals, and thus can be reasoned
independently w.r.t. the TBox and be able to take advantage of existing
parallel-processing techniques.

\section{Empirical Evaluation}
\label{evaluation}

We implemented the rolling-up procedures for computation of
MSC$_\mathcal{T}$'s based on the OWL
API\footnote{http://sourceforge.net/projects/owlapi}, and evaluated
the MSC method for instance checking and retrieving on a lab PC with
Intel(R) Xeon(R) 3.07 GHz CPU, Windows 7, and 1.5 GB Java Heap. 
For the test suite, we have collected a set of well-known ontologies
with large ABoxes:
\begin{enumerate}
\item LUBM(s) (LM) are benchmark ontologies generated using the tool
  provided by \cite{Guo2005},
\item Arabidopsis thaliana (AT) and Caenorhabditis elegans (CE) are
  two  biomedical ontologies\footnote{http://www.reactome.org/download}, 
  sharing a common TBox called \emph{Biopax} that models biological
  pathways, and 
\item DBpedia$^*$ (DP) ontologies are \emph{extended} from the
  original DBpedia ontology \cite{Auer2007}: expressivity of their
  TBox is extended from $\mathcal{ALF}$ to $\mathcal{SHI}$ by adding
  complex roles and concepts defined on role restrictions; their ABoxes
  are obtained by random sampling on the original triple store.
\end{enumerate}
Details of these ontologies can be found in Table \ref{onts}, in terms
of DL expressivity, number of atomic concepts (\# Cpts), TBox
axioms (\# Axms), named individuals (\# Ind.), and ABox
assertions (\# Ast.). 
Notice that, DL expressivity of AT and CE is originally
$\mathcal{SHIN}$, but in our experiments, number restrictions
(i.e $\mathcal{N}$) in their ontology TBox are removed.

\begin{table}[t]
\caption{Information of tested ontologies.}
\label{onts}
\setlength{\tabcolsep}{.9mm}
\begin{center}
\resizebox{0.45\textwidth}{!}{
\begin{tabular}{c|c|c|c|c|c}
\hline
Ontology & Expressivity &  \# Cpts. & \# Axms. & \# Ind. & \# Ast.   
\\ 

\hline
LM1 & $\mathcal{SHI}$ & 43 & 42 &\ 17,175\ &\ 67,465\ 
\\

\hline
LM2 & $\mathcal{SHI}$ & 43 & 42 &\ 78,579\ &\ 319,714\ 
\\

\hline
AT & $\mathcal{SHI}$  & 59 & 344 &\ 42,695\ &\ 117,405\  
\\

\hline
CE & $\mathcal{SHI}$ & 59 & 344 &\ 37,162\ &\ 105,238\ 
\\

\hline
DP1 & $\mathcal{SHI}$ & 449 & 465 &\ 273,663\ &\ 402,062\ 
\\

\hline
DP2 & $\mathcal{SHI}$ & 449 & 465 &\ 298,103\ &\ 419,505\  

\\

\hline
\end{tabular}}
\end{center}
\end{table}

\subsection{Complexity of MSC$_\mathcal{T}$}

Using the MSC (or MSC$_\mathcal{T}$) method, the original instance checking problem is
converted to a subsumption test, the complexity of which
could be computationally high w.r.t. both size of a TBox and size of
the testing concepts \cite{Donini2007}.
Therefore, when evaluating the rolling-up procedure for computation of 
MSC$_\mathcal{T}$'s, one of the most important criteria
is the size of each resultant MSC$_\mathcal{T}$, as it is
the major factor to the time efficiency of a subsumption test,
given a relatively static ontology TBox. 

As we already know, one of the major source of complexity in
ontology reasoning is the so-called \emph{"and-branching"},
which introduces new individuals in the tableau expansion through the
$\exists$-rule, and affects the searching space of the reasoning
algorithm as discussed in \cite{Donini2007}. Thus, when evaluating
sizes of computed MSC$_\mathcal{T}$'s, we measure both the level of
nested quantifiers (i.e. quantification depth) and the number of conjuncts
of each MSC$_\mathcal{T}$.  For example, the concept 
\begin{center}
$\exists R_1.C_1 \sqcap \exists R_2.(C_2 \sqcap \exists R_3.C_3)$
\end{center}
has quantification depth 2 and number of conjuncts 2 (i.e. $\exists
R_1.C_1$  and $\exists R_2.(C_2 \sqcap \exists R_3.C_3)$).

\subsubsection{Experiment Setup}
To evaluate and show efficacy of the proposed strategy and
optimization, we have implemented the following three versions of the
rolling-up method for comparison:  
\begin{LaTeXdescription}
\item[V1.] The original rolling-up procedure adapted to ABox
  assertions without applying the call-by-need strategy,
 which computes the most specific concept
 w.r.t. $\mathcal{A}$ for a given individual.

\item[V2.] The rolling-up procedure with
  the proposed \emph{call-by-need} strategy based on $\tt SYN\_COND$,
  which features the branch pruning as fully discussed in Section \ref{prun}.

\item[V3.] The rolling-up procedure with the \emph{call-by-need}
  strategy based on $\tt SYN\_COND^*$ as discussed in
  Section \ref{opt}. 
\end{LaTeXdescription} 

We compute the MSC$_\mathcal{T}$ for each individual in every
ontology using the three methods respectively, and report in 
Table \ref{sizes-depth} and Table \ref{sizes-conj} the maximum and
the average of quantification depth and number of conjuncts of the
concepts, respectively. We also demonstrate the running-time
efficiency of the optimized rolling-up procedure by showing the
average time spent on computation of a MSC$_\mathcal{T}$ for each
individual in Figure \ref{time}.  

\begin{table}[t]
\caption{Quantification depth of MSC$_\mathcal{T}$'s from different
  rolling-up procedures}
\label{sizes-depth}
\setlength{\tabcolsep}{.9mm}
\begin{center}
\resizebox{0.48\textwidth}{!}{
\begin{tabular}{c| c | c | c } %
\hline
& \multicolumn{1}{c|}{V1}  & \multicolumn{1}{c|}{V2}  &
\multicolumn{1}{c}{V3}  \\
&\multicolumn{1}{c|}{\quad\quad\quad \quad\quad\quad \quad\quad} 
& \multicolumn{1}{c|}{\quad\quad\quad\quad\quad\quad \quad\quad} 
& \multicolumn{1}{c}{\quad\quad\quad \quad\quad\quad \quad\quad}\\ [-2ex]
Ontology & \shortstack{Max. / Avg.} &
\shortstack{Max. / Avg.}  
& \shortstack{Max. / Avg.} \\ 

\hline
LM1 & 5,103 / 4,964.68 & 215 / 98.8 & 2 / 1.48 
\\

\hline
LM2 & 23,015 / 22,654.01 & 239 / 103.59 & 2 / 1.51 
\\

\hline
AT & 2,605 / 2,505.50  & 1,008 / 407.97 & 8 / 3.02
\\

\hline
CE & 3,653 / 3,553.4 & 1,106 / 437.18 & 8 / 2.76 
\\

\hline
DP1 & 3,906 / 3,070.80  & 50 / 2.98 & 4 / 1.13 
\\

\hline
DP2 & 3,968 / 3,865.60 & 58 / 2.94 & 5 / 1.12 
\\

\hline
\end{tabular}}
\end{center}
\end{table}

\begin{table}[t]
\caption{Number of conjuncts of MSC$_\mathcal{T}$'s from different
  rolling-up procedures}
\label{sizes-conj}
\setlength{\tabcolsep}{.9mm}
\begin{center}
\resizebox{0.48\textwidth}{!}{
\begin{tabular}{c| c | c | c } %
\hline
& \multicolumn{1}{c|}{V1}  & \multicolumn{1}{c|}{V2}  &
\multicolumn{1}{c}{V3}  \\
&\multicolumn{1}{c|}{\quad\quad\quad \quad\quad\quad \quad\quad} 
& \multicolumn{1}{c|}{\quad\quad\quad\quad\quad\quad \quad\quad} 
& \multicolumn{1}{c}{\quad\quad\quad \quad\quad\quad \quad\quad}\\ [-2ex]
Ontology & \shortstack{Max. / Avg.} & \shortstack{Max. / Avg.}  & \shortstack{Max. / Avg.} \\ 

\hline
LM1 & 104 / 31.34 & 8 / 3.58 & 4 / 1.56 
\\

\hline
LM2 & 203 / 64.71 & 8 / 3.78 & 4 / 1.59 
\\

\hline
AT & 88 / 87.99 & 19 / 5.92 & 13 / 1.97 
\\

\hline
CE & 52 / 50.70 & 16 / 6.25 & 12 / 1.94  
\\

\hline
DP1 &\ 33,591 / 26,864.50\  & 71 / 5.19 & 12 / 1.83 
\\

\hline
DP2 &\ 60,047 / 60,011.60\  & 64 / 3.95 & 12 / 1.79 
\\

\hline
\end{tabular}}
\end{center}
\end{table}

\begin{figure}[t]
\centering
\includegraphics[width=0.48\textwidth, bb=0 0 500 274]{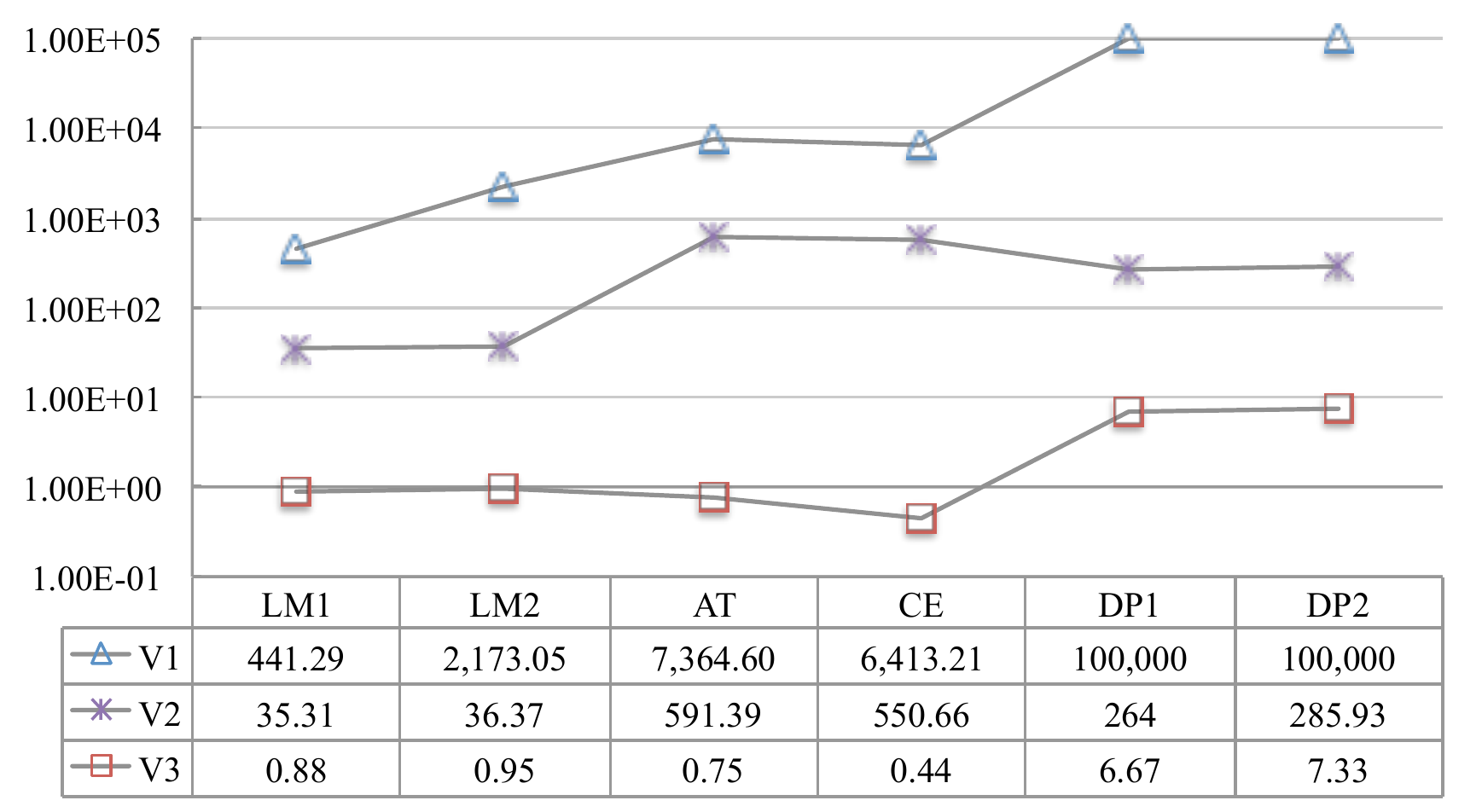}
\caption{Average time (ms) on computation of a
  MSC$_\mathcal{T}$. Timeout is set to be 100,000 ms.}
\label{time}
\end{figure}

\begin{figure}[t!]
\centering
\includegraphics[width=0.48\textwidth, bb=0 -1 344 436]{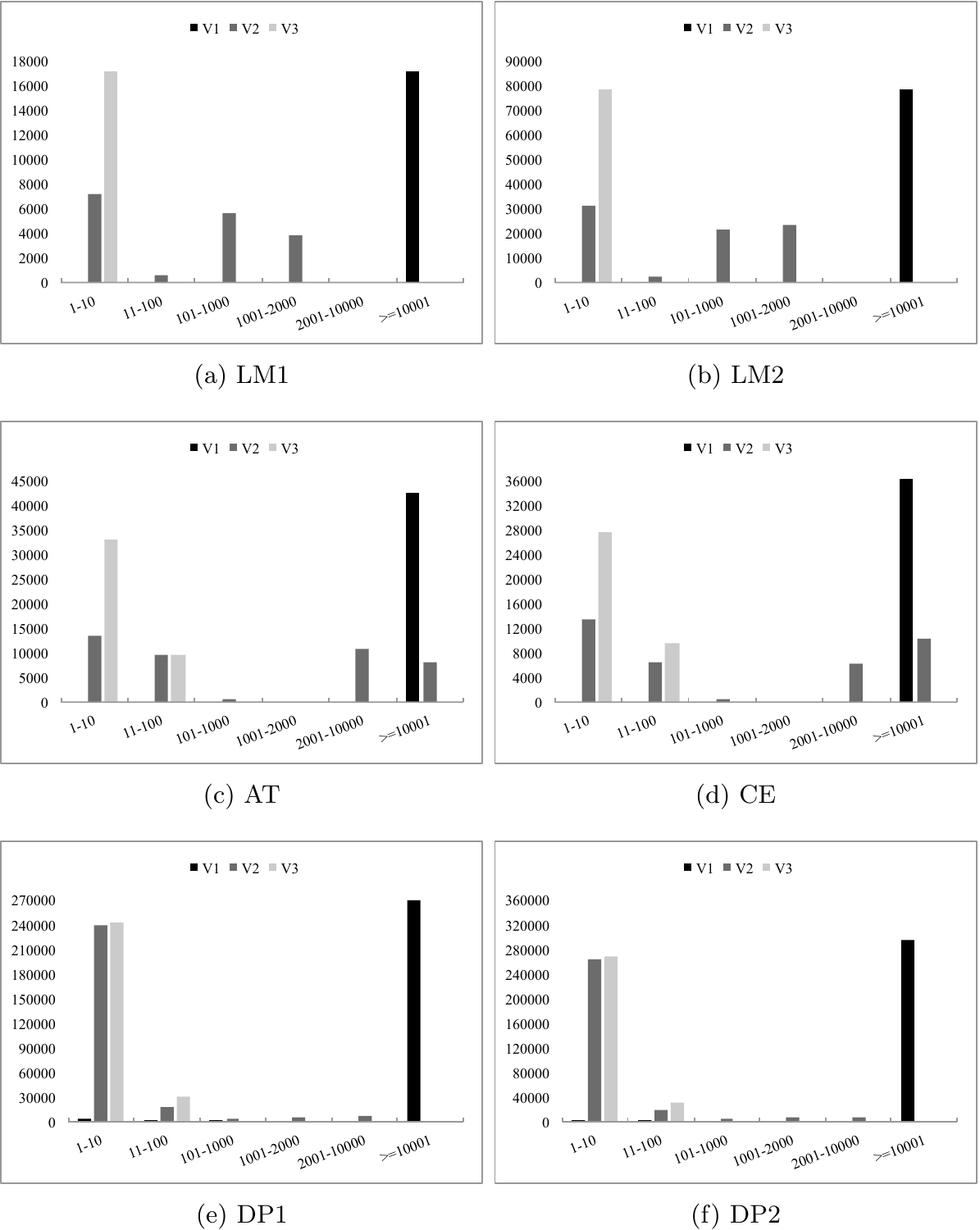}
\caption{Distribution of the number of existential quantifiers in each
  MSC$_\mathcal{T}$.  X-axis gives the size range and the Y-axis gives
  the  number of MSC$_\mathcal{T}$'s that fall in each category.}
\label{sizes}
\end{figure}

\subsubsection{Result Analysis}
As we can see from Table \ref{sizes-depth} and Table \ref{sizes-conj},
the sizes of MSC$_\mathcal{T}$'s generated by V2 and V3 are
significantly smaller than those 
generated by V1 (the original method), which are almost in the same
scale of size of the corresponding ontology ABox. 
The large size of MSC$_\mathcal{T}$'s from V1 is caused by the
fact that most individuals (greater than 99\%) in each of these ontologies 
are connected together by role paths in the graph.  The bulk of each
MSC$_\mathcal{T}$ makes the original MSC method completely inefficient
and unscalable for answering object queries, as a subsumption test
based on these concepts would be prohibitively expensive as a complete
ABox reasoning.  Thus, the comparison here reflects the potential and the
importance of our proposed optimizations in this paper,
which revive the MSC (i.e. MSC$_\mathcal{T}$) method as an efficient way
for instance checking and object query answering. 

The comparison between V2 and V3 demonstrates the efficacy of
the optimization technique discussed in Section \ref{opt}, which could
prevent the rolling-up in arbitrary directions by providing a more rigorous
precondition based on $\tt SYN\_COND$. This optimization
could be useful in many practical ontologies, especially when their ABoxes
contain ``hot-spots'' individuals that connect (tens of) thousands of
individuals together and could cause the rolling-up to
generate concepts with a prohibitive quantification depth.

In particular, in our previous study of modularization for ontology
ABoxes \cite{Xu2013}, the biomedical ontologies (i.e. AT and CE ) are
found to be complex with  
many of their ontology roles (33 out of 55) used for concept
definitions, and their 
ABoxes are hard to be modularized even with various optimization
techniques applied \cite{Xu2013}. 
However, in this paper, we found much simpler MSC$_\mathcal{T}$'s can
also be achieved in these complex ontologies when the optimization
(i.e.  $\tt SYN\_COND^*$) is
applied. For example, the maximum quantification depths of computed
MSC$_\mathcal{T}$'s in both AT and CE are decreased significantly from
more than 1,000 to less than 10.  
Nevertheless, it should also be noted that, effectiveness of this optimization
may vary on different ontologies, depending on their different levels of
complexity and different amount of explicit
information in their ABoxes that can be explored for optimization.   

To further evaluate the complexity of MSC$_\mathcal{T}$'s,
we estimate the number of potential tableau nodes that coud be
created when reasoning with a MSC$_\mathcal{T}$, based on the number
of existential quantifiers in the concept.
We show the distribution of this estimation for MSC$_\mathcal{T}$'s of
each ontology in Figure \ref{sizes}, where the X-axis gives the size
range and the Y-axis gives the number of MSC$_\mathcal{T}$'s that
fall in each category.

\subsection{Reasoning with MSC$_\mathcal{T}$}

In this section, we will show the efficiency that can be
achieved when using the computed MSC$_\mathcal{T}$ for instance
checking and retrieving.
We conduct the experiments on the collected ontologies, and measure
the average reasoning time that is required when performing instance
checking (for every ABox individual)
and instance retrieval using the MSC$_\mathcal{T}$ method, respectively. 

\subsubsection{Experiment Setup}
We will not compare our method with a particular optimization
technique for ABox reasoning, such as lazy unfolding, absorption, or
model merging, etc., since they have already been built into existing
reasoners and it is usually hard to control reasoners to switch on or
off a particular optimization technique. Additionally, the MSC$_\mathcal{T}$
method still relies on the reasoning services provided by the
state-of-art reasoners. Nevertheless, we do compare the reasoning
efficiency between the MSC$_\mathcal{T}$ method and a regular
complete ABox reasoning using existing reasoners, but only to show the 
effectiveness of the proposed MSC$_\mathcal{T}$ method for efficient 
instance checking and data retrieving. 
Moreover, we also compare the MSC$_\mathcal{T}$ method with
the ABox partitioning method (\emph{modular reasoning}) developed
in \cite{Xu2013}, as they are developed based on the similar
principles and both allow parallel or distributed reasoning.

The MSC$_\mathcal{T}$'s here are computed using algorithm V3, and the
ABox partitioning technique used is the most optimized one presented
in \cite{Xu2013}. For a regular complete ABox reasoning, the reasoners
used are OWL DL reasoners, HermiT \cite{Motik2009} and
Pellet \cite{Sirin2007}, 
each of which has its particular optimization techniques implemented
for the reasoning algorithm. Both the MSC$_\mathcal{T}$ method and the
modular reasoning are based on reasoner HermiT, and they
are not parallelized but instead running in an arbitrary
\emph{sequential} order of MSC$_\mathcal{T}$'s or ABox partitions.

{\bf Queries.} LUBM comes with 14 standard queries. For biomedical and
DBpedia$^*$ ontologies respectively, queries listed in Figure \ref{queries}
are used.

\begin{figure}[t]
\centering
\resizebox{0.48\textwidth}{!}{
\begin{tabular}{rcl}
\hline \\ [-1ex]
\multicolumn{3}{l}{$\bf Query\ on\ biomedical\ ontologies:$} \\ [+0.5ex]
$Q1$ & = & $\tt Control\ \sqcap\ \forall controlled.Catalysis\ \sqcap\ \forall
controller.PhysicalEntity$ \\
$Q2$ & = & $\tt Interaction\ \sqcap\ \forall participant.PhysicalEntity$ \\
$Q3$ & = & $\tt Interaction\ \sqcap\ \exists interactionType.\top\ \sqcap\ \exists
participant.\top$ \\
& & $\tt \sqcap\ \forall participant.Gene\ \sqcap\ \exists phenotype.\top$ \\
$Q4$ & = & $\tt PhysicalEntity\ \sqcap\ \forall entityReference.DnaReference$ \\
&& $\tt \sqcap\ \forall memberPhysicalEntity.Dna$ \\
$Q5$ & = & $\tt PhysicalEntity\ \sqcap\ \exists entityReference.SmallMoleculeReference$ \\
&& $\tt \sqcap\ \exists feature.BindingFeature\ \sqcap\ \exists notfeature.BingdingFeature$ \\
&& $\tt \sqcap\ \exists memberPhysicalEntity.SmallMolecule$\\ [+1ex]
\multicolumn{3}{l}{$\bf Query\ on\ DBpedia^*\ ontologies:$} \\ [+0.5ex]
$Q1$ & = & $\tt Person\ \sqcap\ \exists nationality.(Country\ \sqcap\
\exists officialLanguage.Engilish)$ \\
$Q2$ & = & $\tt Music\ \sqcap\ \exists composer.MusicalArtist$ \\
$Q3$ & = & $\tt Person\ \sqcap\ \exists child.Human\ \sqcap\ \exists spouse.Person$ \\
$Q4$ & = & $\tt Event\ \sqcap\ \exists commander.Person\ \sqcap \exists place.City$ \\
$Q5$ & = & $\tt Produce\ \sqcap\ \exists manufacturer.(Organization\ \sqcap\ \exists place.Country)$ \\
\hline
\end{tabular}}
\caption{Queries for biomedical and DBpedia$^*$ ontologies.}
\label{queries}
\end{figure}

For each test ontology, we run the reasoning for each of the given
queries. We report the average reasoning time spent on instance
checking (Figure \ref{ic-time}) and instance
retrieval (Figure \ref{ir-time}), respectively.  
The reasoning time reported here does not include the time spent for
resource initialization (i.e. ontology loading and reasoner initialization),
since the initialization stage can be done offline for query answering.
However, it is obvious that the MSC$_\mathcal{T}$ method should
be more efficient, since it only requires to load an
ontology TBox while a regular ABox reasoning requires to load
an entire ontology (including large ABoxes).
For reasoning with MSC$_\mathcal{T}$'s and ABox partitions, any
updates during the query answering procedure (e.g. update the reasoner 
for different ABox partitions or different MSC$_\mathcal{T}$'s) is
counted into the reasoning time.    

Another point worth noting here is that, for answering object queries
using either modular reasoning or the MSC$_\mathcal{T}$ method, the
overhead (time for ABox modularization or computation of
MSC$_\mathcal{T}$'s) should be taken into account. However, as shown
in previous section and in \cite{Xu2013}, this overhead is negligible
comparing with the efficiency gained on the reasoning, not to mention
when these two methods get parallelized using existing frameworks such
as MapReduce \cite{Dean2008}.

\begin{figure}[t]
\centering
\includegraphics[width=0.48\textwidth, bb=0 0 500 274]{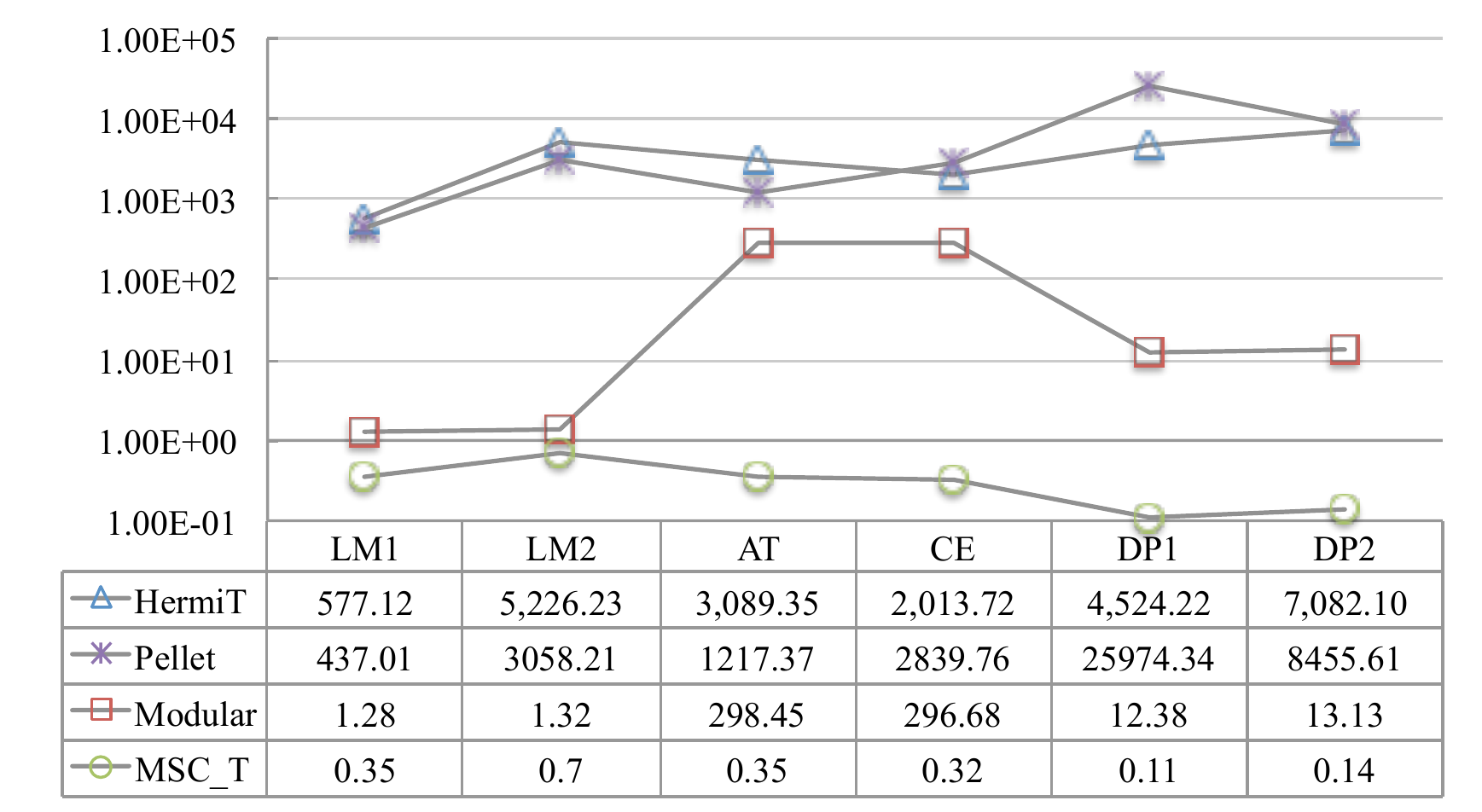}
\caption{Average time (ms) on instance checking.}
\label{ic-time}
\end{figure}

\begin{figure}[t]
\centering
\includegraphics[width=0.48\textwidth, bb=0 0 500 274]{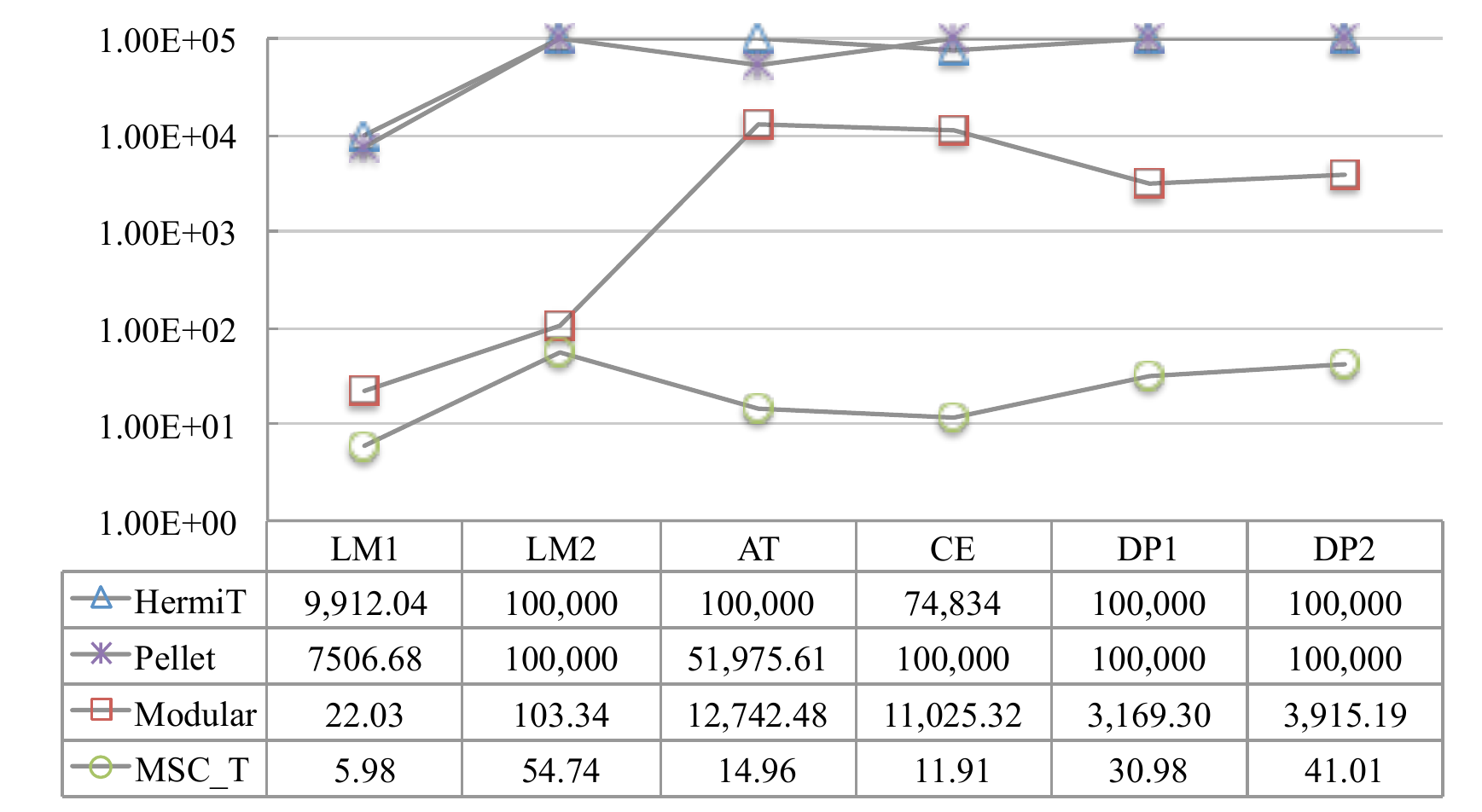}
\caption{Average time (s) spent on instance retrieval. Timeout is set to be 100,000 s.}
\label{ir-time}
\end{figure}

\subsubsection{Result Analysis}
As can be seen from the above two figures, using the MSC$_\mathcal{T}$
method, reasoning efficiency for both instance checking and instance
retrieval in the testing ontologies has been improved significantly:
($\romannumeral 1$) by more than three orders of magnitude when
comparing with a complete reasoning; ($\romannumeral 2$) and by about
two orders of magnitude (except in LUBM1 and LUBM2) when comparing with 
the modular reasoning. For the latter, the improvement 
in LUBM1 and LUBM2 are not as significant as in others,
which is because of the simplicity of these two ontologies that allows
fine granularity of ABox partitions to be achieved \cite{Xu2013}.

On the other hand, using the MSC$_\mathcal{T}$ method in complex
ontologies, such as AT and CE, the great improvement in reasoning
efficiency comes from the reduction of searching space for reasoning
algorithms, by branch pruning and also \emph{concept absorption}
during the computation of MSC$_\mathcal{T}$'s. For example, consider an
individual $x$ having the following $n$ role assertions:
\begin{center}
$R(x,y_1), \ R(x, y_2), \cdots, R(x, y_n)$,
\end{center} 
where $y_i \in D$ and  $n$ tends to be large in these practical ontologies.
Rolling up these assertions may generate a set of $\exists R.D$'s, the
conjunction of which is still $\exists R.D$. Thus, when using this concept
for instance checking, the interpretation may generate only one
$R$-neighbor of individual $x$ instead of $n$.

\subsection{Scalability Evaluation}
Using the MSC$_\mathcal{T}$ method for query answering over large
ontologies is intended for distributed (parallel) computing. However,
even if it is executed sequentially in a single machine, linear
scalability may still be achieved on large ontologies that are not extremely
complex; and there are mainly two reasons for that: first, the computation of
MSC$_\mathcal{T}$'s focuses on only the query-relevant
assertions instead of the entire ABox; second, the obtained
MSC$_\mathcal{T}$'s could be very simple, sizes of which could be significantly
smaller than that of the ABoxes.
We test the scalability of this method for query answering
(sequentially executed) using the benchmark ontology LUBM, which
models organization of universities with each university constituted about
17,000 related individuals. The result is show in Figure \ref{scalability}.

\begin{figure}[t]
\centering
\includegraphics[width=0.48\textwidth, bb=0 0 403 211]{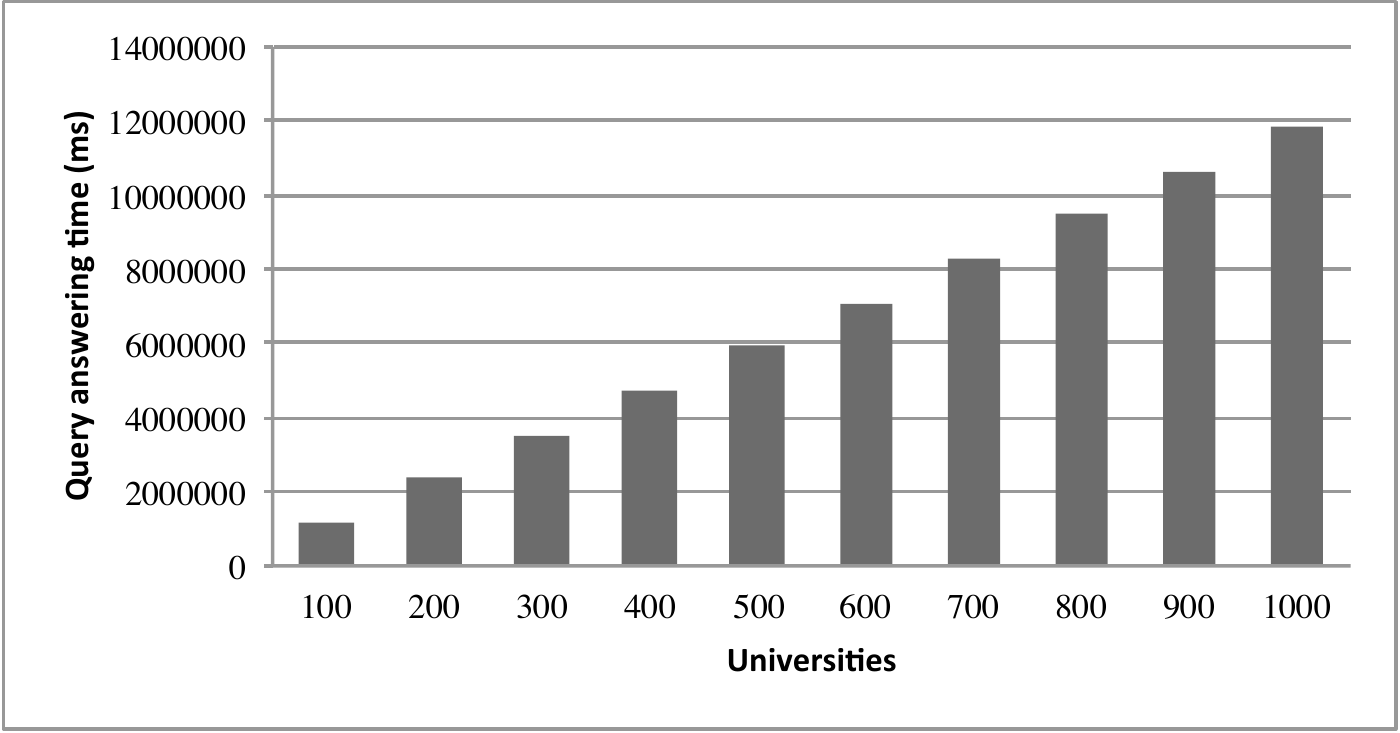}
\caption{Scalability evaluation}
\label{scalability}
\end{figure}

\section{Conclusion and Outlook}
\label{con}

In this paper, we proposed a revised MSC method for efficient instance
checking. This method allows the ontology reasoning to
explore only a much smaller subset of ABox data that is relevant to a
given instance checking problem,  thus being able to achieve great
efficiency and to solve the limitation of current memory-based
reasoning techniques. It can be particularly useful for answering
object queries over those large \emph{non-Horn} DL ontologies, where
existing optimization techniques may fall short and answering object
queries may demand thousands or even millions of instance checking
tasks. Most importantly, due to the independence between
MSC$_\mathcal{T}$'s, scalability for query answering over huge ontologies
(e.g. in the setting of semantic webs) could also be achieved by
parallelizing the computations. 

Our technique currently works for logic $\mathcal{SHI}$, which is
semi-expressive and is sufficient for many of the practical
ontologies. However, the use of more expressive logic in modeling
application domains requires more advanced technique for efficient
data retrieving from ontology ABoxes. For the future work, we will
investigate on how to extend the current technique to support
$\mathcal{SHIN}$ or $\mathcal{SHIQ}$ that are featured with
(qualified) number restrictions. We will concentrate on extending
the rolling-up procedure to generate number restrictions, such as 
$\geq nR.\top$ or $\geq nR.C$, whenever there is a need. 
We will also have to take a particular care of the identical individual
problem, where concepts and role assertions of an individual can be
derived via individual equivalence.

\ifCLASSOPTIONcaptionsoff
  \newpage
\fi



\bibliographystyle{IEEEtran}
\bibliography{abox}



%

\end{document}